\documentclass[11pt]{article}
\usepackage[utf8]{inputenc}
\usepackage{amsfonts}
\usepackage{amssymb}
\usepackage[table]{xcolor}

\usepackage{fullpage,tikz,enumitem}
\usetikzlibrary{matrix, fit}
\usetikzlibrary{backgrounds}
\usepackage{caption}
\usepackage{subcaption}

\usepackage[colorlinks=true, pdfstartview=FitV, 
linkcolor=black, citecolor=black, 
urlcolor=black, plainpages=false,
pdfpagelabels]{hyperref}

\usepackage{bookmark}
\usepackage{amsmath,amsthm,amssymb,verbatim,graphicx,fullpage,url}
\usepackage{mathtools}
\usepackage{thm-restate}
\usepackage{thmtools}

\usepackage[capitalise]{cleveref}

\def\1{\mathbf{1}} 
\def\0{\mathbf{0}}

\usepackage{amsmath,amsthm, amssymb,amscd,amstext,amsfonts}
\usepackage[utf8]{inputenc}
\usepackage{mathtools}
\usepackage{mathrsfs}
\usepackage{thmtools}
\usepackage{latexsym}
\usepackage{verbatim}
\usepackage{framed}
\usepackage{graphicx}
\usepackage{stmaryrd}
\usepackage{enumitem}
\usepackage{fullpage}
\usepackage{bm}
\usepackage{url}
\usepackage{physics}
\usepackage{dsfont}
\usepackage{todonotes}

\usepackage{mathtools}

\theoremstyle{plain}

\newtheorem{thm}{Theorem}[section]
\newtheorem{theorem}[thm]{Theorem}

\newtheorem{corollary}[thm]{Corollary}
\newtheorem{lemma}[thm]{Lemma}

\newtheorem{definition}[thm]{Definition}

\newtheorem{problem}[thm]{Problem}
\newtheorem{question}[thm]{Question}
\newtheorem{claim}[thm]{Claim}
\newtheorem*{claim*}{Claim}

\newtheorem{remark}[thm]{Remark}

\renewcommand{\abs}[1]{|#1|}      
\newcommand{\absB}[1]{\left|#1\right|}

\newcommand{\set}[1]{\{#1\}}

\DeclareMathOperator{\bp}{bp}

\newcommand{\N}{\mathbb{N}}

\newcommand{\HH}{\mathbb{H}}
\def\barHH{\overline\HH}

\newcommand{\cX}{\mathcal X}

\DeclareMathOperator{\VC}{VC}       
\DeclareMathOperator{\LD}{LD}

\def\SOA{\textrm{SOA}}

\newcommand{\defeq}{\coloneqq}

\DeclareMathOperator{\Cer}{C}
\DeclareMathOperator{\UC}{UC}
\DeclareMathOperator{\Cov}{Cov}
\DeclareMathOperator{\UCov}{UCov}

\usepackage{ifthen}
\newcommand{\agnorm}[2][]{
	\ifthenelse{\equal{#2}{}}{
		\widetilde{\gamma}_2^{#1}
	}{
		\widetilde{\gamma}_2^{#1}(#2)
	}
}

\DeclareFontFamily{U}{mathx}{}
\DeclareFontShape{U}{mathx}{m}{n}{<-> mathx10}{}
\DeclareSymbolFont{mathx}{U}{mathx}{m}{n}
\DeclareMathAccent{\widecheck}{0}{mathx}{"71}

\def\calF{\mathcal{F}}

\title{Online Learning and Disambiguations of Partial Concept Classes}

\author{
	Tsun-Ming Cheung\thanks{
		School of Computer Science, McGill University. Email:
		\texttt{tsun.ming.cheung@mail.mcgill.ca}.} 
	\and 
	Hamed Hatami \thanks{School of Computer Science, McGill University. Email: \texttt{hatami@cs.mcgill.ca}. Supported by an NSERC grant.}
	\and Pooya Hatami \thanks{Computer Science and Engineering, Ohio State University. \texttt{pooyahat@gmail.com}. Supported by NSF grant CCF-1947546} 
    \and Kaave Hosseini \thanks{Department of Computer Science, University of Rochester. Email: \texttt{kaave.hosseini@rochester.edu}}	
}

\date{}

\begin{document}

\maketitle
\begin{abstract}
In a recent article, Alon, Hanneke, Holzman, and Moran (FOCS '21) introduced a unifying framework to study the learnability of classes of  \emph{partial} concepts. One of the central questions studied in their work is whether the learnability of a partial concept class is always inherited from the learnability of some ``extension'' of it to a total concept class. 

They showed this is not the case for PAC learning but left the problem open for the stronger notion of online learnability. 

We resolve this problem by constructing a class of partial concepts that is online learnable, but no extension of it to a class of total concepts is online learnable (or even PAC learnable).    

\end{abstract} 
\section{Introduction}
In many practical learning problems, the learning task is tractable because we are only required to predict the labels of the data points that satisfy specific properties.
In the setting of binary classification problems, instead of learning a total concept $h:\cX \to \{0,1\}$,  we are often content with learning a partial version of it $\widetilde{h}:\cX \to \{0,1,\star\}$, where $\widetilde{h}(x)=\star$ means that both $0$ and $1$ are acceptable predictions. This relaxation of allowing unspecified predictions renders a wider range of learning tasks tractable.

Consider, for example, predicting whether a person approves or disapproves of various political stances by observing their previous voting pattern. This person might not hold a strong opinion about particular political sentiments, and it might be impossible to predict their vote on those issues based on their previous history. However, the learning task might become possible if we allow both  ``approve'' and ``disapprove''  as acceptable predictions in those cases where a firm conviction is lacking. 

A well-studied example of this phenomenon is learning half-spaces with a large margin. In this problem, the domain is the set of points in a bounded region in an arbitrary Euclidean space, and the concepts are half-spaces that map each point to $1$ or $0$ depending on whether they belong to the half-space or not. It is well-known that when the dimension of the underlying Euclidean space is large, one needs many samples to learn a half-space. However, in the large margin setting, we are only required to correctly predict the label of a point if its distance from the defining hyperplane is bounded from below by some margin. Standard learning algorithms for this task, such as the classical Perceptron algorithm, due to  Rosenblatt \cite{Rosenblatt1958ThePA}, show that this relaxation of the  learning requirement  makes the   problem tractable even for high-dimensional Euclidean spaces. 
Motivated by such examples, Alon, Hanneke, Holzman, and Moran~\cite{AHHM21} initiated a systematic study of the learnability of partial concept classes $\HH \subseteq \{0,1,\star\}^\cX$. They focused  on the two frameworks of \emph{probably approximately correct (PAC) learning} and \emph{online learning}.  We refer to~\cite{AHHM21} for the definition of PAC learnability of partial concept classes. We define online learnability in  \cref{def:Online}.

PAC learning is an elegant theoretical framework characterized by the combinatorial parameter of the Vapnik–Chervonenkis ($\VC$) dimension. The fundamental theorem of PAC learning states that a total binary concept class is PAC learnable if and only if its  $\VC$ dimension is finite.  Similarly, online learnability of total concept classes is characterized by a combinatorial parameter called the Littlestone dimension ($\LD$). We formally define the $\VC$ dimension and the Littlestone dimension in \cref{def:VC,def:LD} respectively.  
Alon, Hanneke, Holzman, and Moran~\cite{AHHM21} proved that these characterizations of PAC and online learnability  extend to the setting of partial concept classes. 

\begin{theorem}[{\cite[Theorems 1 and 15]{AHHM21}}]
\label{thm:LearnPartial}
Let  $\HH \subseteq \{0,1,\star\}^\cX$ be a partial concept class. 
\begin{itemize}
\item $\HH $ is PAC learnable if and only if $\VC(\HH) < \infty$. 
\item  $\HH $ is online learnable if and only if $\LD(\HH) < \infty$. 
\end{itemize}
\end{theorem}

It follows from the definitions of VC and LD dimensions that for every partial concept class $\HH \subseteq \{0,1,\star\}^\cX$, we have $\VC(\HH) \le \LD(\HH)$. In particular, online learnability always implies PAC learnability.

One of the central questions studied in~\cite{AHHM21}  is whether the learnability of a partial concept class  is always inherited from the learnability of some total concept class. To make this question precise, we need to define the notion of disambiguation of a partial concept class. While we defer the formal definitions to \cref{sec:disamb}, one may understand a \emph{strong disambiguation} of a partial class as simply an assignment of each $\star$ to either 1 or 0 for each partial concept in the class. 
When $\mathcal{X}$ is infinite, it is more natural to consider the weaker notion of \emph{disambiguation} that we shall define in \cref{def:Disamb}. When $\cX$ is finite, the notions of disambiguation and strong disambiguation coincide.

Consider the problem of learning the partial concept class  $\HH\subseteq \{0,1,\star\}^\mathcal{X}$ in PAC learning or online learning.  If the partial concept class $\HH$ has a   disambiguation $\barHH\subseteq \{0,1\}^\mathcal{X}$ that is PAC learnable, then $\HH$ is PAC learnable. This follows from $\VC(\HH) \le \VC(\barHH)$, or simply by running the PAC learning algorithm of $\barHH$ on $\HH$.  Similarly,  if a disambiguation $\barHH$ of $\HH$ is online learnable, then $\HH$ is online learnable.

Is the learnability of every partial concept class    inherited from the learnability of some   disambiguation to a total concept class? 

\begin{question}[Informal~\cite{AHHM21}]\label{question:q1}
Does every   learnable partial class have a learnable disambiguation?
\end{question}
 
Equipped with the VC dimension characterization of \cref{thm:LearnPartial}, \cite{AHHM21} proved that for PAC learning, the answer to \cref{question:q1} is \emph{negative}.

\begin{theorem}[{\cite[Theorem 11]{AHHM21}}]\label{theorem:alonVCLD}
For every $n \in \N$, there exists a partial concept class $\HH_n \subseteq\{0,1,\star\}^{[n]}$ with $\VC(\HH_n) = 1$  such that any disambiguation $\barHH$ of $\HH_n$
has $\VC(\barHH) \geq (\log n)^{1-o(1)}$. Moreover, for  $\mathcal{X} = \mathbb{N}$, there exists $\HH_\infty \subseteq \{0,1,\star\}^\mathcal{X}$ with $\VC(\HH_\infty) = 1$ such that $\VC(\barHH) = \infty$ for every disambiguation $\barHH$ of $\HH_\infty$. 
\end{theorem}

While \Cref{theorem:alonVCLD} gives a strong negative answer to \cref{question:q1} in the case of PAC learning,  the question was left   open   for online learning. Roughly speaking, this question strengthens the bounded-$\VC$ assumption on $\HH$ to bounded \emph{Littlestone dimension} ($\LD$), which pertains to \emph{online learnability} of $\HH$.

The authors in~\cite{AHHM21} also proposed a second open problem that replaces the bounded-$\VC$ dimension assumption by the assumption of \emph{polynomial growth}. This assumption is weaker than bounded $\LD$ but stronger than bounded $\VC$ dimension.

As we discuss below, our main result resolves these two open problems. 

\paragraph{Online learnability.}
 Online learning is performed in a sequence of consecutive rounds, where at round
$t$, the learner is presented with an instance $x_t \in \cX$ and is required to predict its label. After predicting the label, the correct label $y_t \in \{0,1\}$ is revealed to the learner. Note that even for partial concept classes, we require that the correct label is $0$ or $1$. The learner's goal is to make as few prediction mistakes as possible during this process. We assume that the true labels are always \emph{realizable}, i.e. there is a partial concept $h \in \HH$ with $h(x_i)=y_i$ for all $i=1,\ldots,t$. 

\begin{definition}[Online Learnability]
\label{def:Online}
A partial concept class $\HH \subseteq \{0,1,\star\}^\cX$ is online learnable if there is a \emph{mistake bound}  $m \defeq m(\HH) \in \mathbb{N}$ such that  for every $T \in \mathbb{N}$, there exists a learning algorithm that on every realizable sequence $(x_i,y_i)_{i=1,\ldots, T}$ makes at most $m$ mistakes.  
\end{definition}

Online learnability for total classes is   equivalent to the bounded Littlestone dimension. In \cref{thm:LearnPartial}, Alon, Hanneke, Holzman, and Moran \cite{AHHM21} showed that the same equivalence carries out in the setting of partial classes. They asked the following  formulation of \cref{question:q1}.

\begin{center}
\emph{If a partial class is online learnable, is there a disambiguation of it that is online learnable?}
\end{center}
More precisely, they pose the following question:
\begin{problem}[\cite{AHHM21}]\label{problem:online}
Let $\HH$ be a partial class with $\LD(\HH) <\infty$. Does there exist a disambiguation $\barHH$ of $\HH$  with $\LD(\barHH) < \infty$? Is there one with $\VC(\barHH) < \infty$?
\end{problem} 
We give a negative answer to \Cref{problem:online}:
\begin{theorem}[Main Theorem]\label{theorem:mainLDVC}
    For every $n \in \N$, there exists a partial concept class $\HH_n \subseteq\{0,1,\star\}^{[n]}$ with $\LD(\HH_n) \leq 2$  such that every disambiguation $\barHH$ of $\HH_n$ satisfies $\LD(\barHH)\geq \VC(\barHH) = \Omega(\log\log n).$ Consequently, for $\mathcal{X} = \mathbb{N}$, there exists $\HH_\infty \subseteq \{0,1,\star\}^\mathcal{X}$ with $\LD(\HH_\infty) \leq 2$ and $\LD(\barHH)\geq \VC(\barHH) = \infty$ for every disambiguation $\barHH$ of $\HH_\infty$. 
\end{theorem}

\paragraph{Polynomial growth.}
A general strategy to prove a super-constant lower bound on the $\VC$ dimension of a total concept class $\HH \subseteq \{0,1\}^n$ is to show that the class is of super-polynomial size. This is the approach utilized in \Cref{theorem:alonVCLD} and \Cref{theorem:mainLDVC}. For a total concept class $\HH \subseteq \{0,1\}^{n}$ with $\VC$ dimension $d$, one has $2^d \leq|\HH|\leq  O(n^d)$: the lower bound is immediate from the definition of $\VC$ dimension, and the upper bound is the consequence of the celebrated Sauer-Shelah-Perles (SSP) lemma.

\begin{theorem}[Sauer-Shelah-Perles lemma~\cite{sauer1972density}]\label{lemma:SSPlemma}
    Let $\HH\subseteq\{0,1\}^n$ and $\VC(\HH) = d$. Then \[|\HH|\leq \binom{n}{\leq d}\defeq  \sum^d_{i=0} \binom{n}{i}=O(n^d).\]
\end{theorem}

The direct analog of the SSP lemma is not true for partial concept classes: \cite{AHHM21} proved that there exists $\HH\subseteq\{0,1,\star\}^{[n]}$  with $\VC(\HH)=1$ such that every disambiguation $\barHH$  has size  $|\barHH|\geq n^{\Omega(\log n)}$.
This result, combined with the SSP lemma for total classes, immediately implies \Cref{theorem:alonVCLD}. 
 
 Interestingly, under the stronger assumption of the bounded Littlestone dimension, the polynomial growth behavior of the original SSP lemma remains valid. 

\begin{theorem}[\cite{AHHM21}]\label{theorem:partialSSP}
Every partial concept class  $\HH\subseteq\{0,1,\star\}^{[n]}$   with $\LD(\HH) \le d$ has a disambiguation $\barHH$ with $|\barHH|\leq O(n^{d})$. 
\end{theorem}

We say that a partial concept class $\HH\subseteq\{0,1,\star\}^{\cX}$ has \emph{polynomial growth with parameter $d\in \N$} if  for every finite $\mathcal{X}'\subseteq\mathcal{X}$, there is a disambiguation $\overline{\HH|_{\mathcal{X}'}}$ of $\HH|_{\mathcal{X}'}$ of size at most $O(|\mathcal{X}'|^d)$.  
Note that by \cref{theorem:partialSSP}, every partial concept class with  Littlestone dimension $d$ has polynomial growth with parameter $d$.

  Alon, Hanneke, Holzman, and Moran asked the following question:
 
\begin{problem}[\cite{AHHM21}]\label{problem:polygrowth}
    Let $\HH\subseteq\{0,1,\star\}^\mathcal{X}$ be a partial concept class with polynomial growth. Does there exist a disambiguation $\barHH$ of $\HH$ such that $\VC( \barHH ) < \infty$?
\end{problem}

Note that \cref{problem:polygrowth} cannot be resolved (in the negative) by a naive application of the SSP lemma to disambiguations of $\HH$ or its restrictions. However,   \Cref{theorem:mainLDVC} combined with \Cref{theorem:partialSSP} refutes \Cref{problem:polygrowth} as well.
\begin{theorem}\label{theorem:mainPolygrowth}
    For every $n\in\N$, there is $\HH\subseteq\{0,1,\star\}^{[n]}$  with polynomial growth with parameter $2$ such that  every disambiguation $\barHH$ of $\HH$ has $\VC( \barHH ) = \Omega(\log\log n)$. 
    
    Consequently, for $\mathcal{X} = \mathbb{N}$, there exists $\HH_\infty \subseteq \{0,1,\star\}^\mathcal{X}$ with polynomial growth with parameter $2$ such that every disambiguation $\overline{\HH_\infty}$ of $\HH_\infty$ has $\VC( \overline{\HH_\infty} ) = \infty$. 
\end{theorem}

\paragraph{The Alon-Saks-Seymour Problem.}

The proof of \Cref{theorem:alonVCLD} in~\cite{AHHM21} hinges on the   breakthrough result  of  G{\"o}{\"o}s~\cite{MR3473357} and its subsequent improvements~\cite{kasparsASS} that led to almost optimal super-polynomial bounds on the ``biclique partition number versus chromatic number'' problem of Alon, Saks, and Seymour. The \emph{biclique partition number} of a graph $G$, denoted by $\bp(G)$, is the smallest number of complete bipartite graphs (bicliques) that partition the edge set of $G$. Alon, Saks, and Seymour conjectured  that the chromatic number of a graph with biclique partition number $k$ is at most $k + 1$.   Huang and Sudakov refuted the  Alon-Saks-Seymour conjecture in~\cite{huang2012counterexample} by establishing a superlinear gap between the two parameters.  Later in a   breakthrough, G\"o\"os~\cite{MR3473357} proved a superpolynomial separation.

Our main result, \cref{theorem:mainLDVC}, also builds on the aforementioned graph constructions. However, unlike previous works, our theorem demands a reasonable upper bound on the number of vertices. Since the constructions result from a complex sequence of reductions involving query complexity, communication complexity, and graph theory~\cite{Bousquet2014,MR3473357,MR3561784,kasparsASS}, it is necessary to scrutinize them to ensure that the required parameters are met. We present a reorganized and partly simplified sequence of constructions in \Cref{section:ASS} that establishes the following theorem. 

\begin{theorem}[Small-size refutation of the Alon-Saks-Seymour conjecture]\label{theorem:ASS}
There exists a graph $G$ on $2^{\Theta(k^4 \log^3 k)}$ vertices that admits a biclique partition of size $2^{O(k \log^4 k)}$ but its chromatic number is at least  $2^{\Omega(k^2)}$.
\end{theorem} 
\cref{theorem:ASS} is essentially due to~\cite{kasparsASS}.  Our contribution to this theorem is obtaining an explicit and optimized bound on the size of $G$.

\paragraph{Standard Optimal Algorithm.} 
\cref{theorem:mainLDVC} provides an example partial class with Littlestone dimension $\leq 2$, such that the VC dimension of every disambiguation is $\Omega(\log\log n)$.  Whether one can improve the $\Omega(\log\log n)$ lower bound is unclear. In particular, it is an interesting question whether every disambiguation of a partial class of Littlestone dimension at most $2$ has VC dimension $O(\log\log n)$.  One natural candidate approach for obtaining such an upper bound would be to utilize the Standard Optimal Algorithm (SOA).  

 SOA is an online learning algorithm devised by Littlestone~\cite{Littlestone1988}  that can learn classes with bounded Littlestone dimensions.   Alon, Hanneke, Holzman, and Moran, in their proof of \Cref{theorem:partialSSP}, showed that applying SOA to a partial concept class $\HH$ with Littlestone dimension $d$ yields a disambiguation of size  $|\barHH|\leq O(n^{d})$ and consequently $\VC$ dimension $O(d\log n)$. This shows that the lower bound of \cref{theorem:mainLDVC} on VC dimension of disambiguations cannot be improved beyond $O(\log n)$. It is hence natural to ask whether it is possible to obtain an improved upper bound on the VC dimension of the SOA-based disambiguation.

We answer this question in the negative by constructing a family of partial concept classes $\HH$ of Littlestone dimension $d$ where the  disambiguation obtained by the SOA algorithm has $\VC$ dimension $\Omega(d\log (n/d))$.
\begin{restatable}{theorem}{LDdisamb}\label{theorem:LD-disamb}

    For every natural numbers $d\leq n$, there exists a partial concept class $\HH_{n,d} \subseteq\{0,1,\star\}^{[n]}$ with $d\leq \LD(\HH_{n,d})\leq d+1$ such that the SOA disambiguation of  $\HH_{n,d}$ has $
    \VC$ dimension $\Omega(d\log (n/d))$. 
\end{restatable}
 
\section{Preliminaries and Background}
For a positive integer $k$, we denote $[k]\defeq \set{1,\ldots, k}$. We adopt the convention that $\set{0,1}^0$ or $\set{0,1,\star}^0$ contains the empty string only, which we denote by $()$.

We adopt the standard computer science asymptotic notations, such as Big-O, and use the asymptotic tilde notations to hide poly-logarithmic factors.

\subsection{VC Dimension and Littlestone Dimension} \label{sec:partial}
Let $\HH\subseteq \{0,1,\star\}^\mathcal{X}$ be a partial concept class. When the domain $\cX$ is finite, we sometimes view $\HH$ as a partial matrix $\mathbf{M}_{\mathcal{X}\times \HH}$, where each row corresponds to a point $x\in \mathcal{X}$ and each column corresponds to a concept $h\in \HH$, and the entries are defined as $\mathbf{M}(x,h) = h(x)$.

Next, we define the $\VC$ dimension and the Littlestone dimension of partial classes, which generalize  the definitions of these notions for total classes. As shown in \cite{AHHM21}, the  $\VC$ and Littlestone dimensions for partial classes capture PAC and online learnability, respectively.

\begin{definition}[Shattered set]
    A finite \emph{set} of points $C=\set{x_1,\ldots,x_n} \subseteq \mathcal{X}$ is \emph{shattered} by a partial concept class $\HH\subseteq \{0,1,\star\}^\mathcal{X}$ if for every  pattern $y\in \set{0,1}^n$, there exists $h\in \HH$ with $h(x_i)=y_i$ for all $i\in [n]$.
\end{definition}

\begin{definition}[$\VC$ dimension] \label{def:VC}
    The $\VC$ dimension of a partial class $\HH$, denoted by $\VC(\HH)$, is the maximum $d$ such that there exists a size-$d$ subset of $\cX$ that is  shattered by $\HH$.
    If no such largest $d$ exists, define $\VC(\HH)=\infty$.
\end{definition}
Viewed as a matrix, the $\VC$ dimension of $\HH$ is the maximum $d$ such that the associated partial matrix $\mathbf{M}_{\cX\times \HH}$ contains  a zero/one submatrix of dimensions $d\times 2^d$, where the columns enumerate all $d$-bit zero/one patterns.

The Littlestone dimension is defined through the shattering of decision trees instead of sets. Consider a full binary decision tree of height $d$ where every non-leaf $v$ is labelled with an element $x_v \in \cX$.  We identify every node of this tree by the  string  $v \in \bigcup^{d}_{k=0}\set{0,1}^k$ that corresponds to the path from the root to the node. That is, the root is  the empty string, its children are the two elements in $\{0,1\}$, and more generally, the children of a node $\vec{v} \in \{0,1\}^k$ are the two strings $\vec{v}0$ and $\vec{v}1$ in $\{0,1\}^{k+1}$.  

We say that such a tree is \emph{shattered}  by a partial concept class $\HH$ if for every leaf $y\in \{0,1\}^d$, there exists $h\in \HH$ such that $h(x_{y[<i]})=y_i$ for each $i\in[d]$, where $y[<i]$ is the first $(i-1)$-th bits of $y$. In other words,  applying the decision tree to $h$ will result in the leaf $y$.  
    
\begin{definition}[Littlestone dimension] \label{def:LD}
The \emph{Littlestone dimension} of a partial concept class $\HH$, denoted by $\LD(\HH)$, is the maximum $d$ such that there is an $\cX$-labelled height-$d$ full binary decision tree that is  shattered by $\HH$.  If no such largest $d$ exists, define $\LD(\HH)=\infty$. 
\end{definition}

The \emph{dual} of a concept class $\HH$ is the concept class with the roles of points and concepts exchanged. Concretely, the dual class of   $\HH\in \set{0,1,\star}^\cX$, denoted by $\HH^\top$, is the collection of functions $f_x:\HH \to \set{0,1,\star }$ for every $x\in \cX$, which is defined by $f_x(h)=h(x)$ for each $h\in \HH$. When $\cX$ is finite, taking the dual corresponds to transposing the matrix of the concept class. The VC-dimension of the dual-class is related to that of the primal class by the inequality 
\[\VC(\HH^\top)\leq 2^{\VC(\HH)+1}-1\]
(see \cite{Mat02}), which translates to a lower bound of the VC-dimension of the primal class.  

\subsection{Disambiguations}\label{sec:disamb}  
We start by formally defining  \emph{strong disambiguation} and \emph{disambiguation}. As mentioned earlier, the two notions coincide when the domain $\cX$ is finite.

\begin{definition}[Strong Disambiguation]
   A   \emph{strong disambiguation} of a partial concept class $\HH\subseteq \{0,1,\star\}^\mathcal{X}$ is a total concept class $\barHH\subseteq \{0,1\}^\mathcal{X}$ such that for every $h\in \HH$, there exists a   $\bar{h}\in \barHH$ that is consistent with $h$ on the points $h^{-1}(\{0,1\})$. 
\end{definition}
\begin{definition}[Disambiguation]
\label{def:Disamb}
   A \emph{disambiguation} of a partial concept class $\HH\subseteq \{0,1,\star\}^\mathcal{X}$ is a total concept class $\barHH\subseteq \{0,1\}^\mathcal{X}$ such that for every $h\in \HH$ and every finite $S \subseteq h^{-1}(\{0,1\})$, there exists  $\bar{h}\in \barHH$ that is consistent with $h$ on $S$.
\end{definition} 
 
A learning algorithm can often provide a disambiguation of a partial concept class  by assigning the prediction of the algorithm to unspecified values. Relevant to our work is the disambiguation by the Standard Optimal Algorithm of Littlestone. It was observed in \cite{AHHM21} that this algorithm can provide ``efficient'' disambiguations of partial classes with bounded Littlestone dimensions. We describe this   disambiguation next. 

Consider a partial concept class $\HH\subseteq \set{0,1,\star}^\cX$ with a countable domain $\cX$ and an ordering $x_1,x_2,\ldots$ of $\cX$. Given $\vec{b}\in\set{0,1,\star}^k$, let $\HH|_{\vec{b}}$ be the set of concepts $h$ where $h(x_i)=b_i$ for every $i\in [k]$. For convenience, we identify $\HH|_{()}=\HH$. For the purpose of the algorithm, we adopt the convention $\LD(\emptyset)=-1$. 

The SOA obtains a disambiguation iteratively and assigns a $0/1$ value to each $\star$ in $\HH$: for each $k\in\N$, consider $\HH|_{\vec{b}}$ for every $\vec{b}\in\set{0,1}^{k-1}$. Pick $c\in\set{0,1}$ which maximizes $\LD(\HH|_{\vec{b}c})$, breaking ties by favoring $c=0$, and assign $c$ to $h(x_k)=\star$ for every $h\in \HH|_{\vec{b}\star}$.

We use the notation $\barHH^{\SOA}$ for the SOA disambiguation of a partial concept class $\HH$. As mentioned earlier, for a partial class with Littlestone dimension $d$, \Cref{theorem:partialSSP} gives an upper bound of $\binom{n}{\leq d} = O(n^{d})$ on $\absB{\barHH^{\SOA}}$. The theorem follows from the mistake bound of SOA for online learning, which relies on the crucial property that at least one choice of $c\in\set{0,1}$ satisfies $\LD(\HH|_{\vec{b}c})\leq \LD(\HH|_{\vec{b}})-1$ whenever $\HH|_{\vec{b}}\neq \emptyset$. 

\section{Proofs}
In this section, we present the proofs of \Cref{theorem:mainLDVC,theorem:mainPolygrowth,,theorem:ASS,theorem:LD-disamb}. 

\subsection{Proofs of \Cref{theorem:mainLDVC,theorem:mainPolygrowth}}
 As mentioned earlier, \Cref{theorem:mainPolygrowth} is an immediate corollary of \Cref{theorem:mainLDVC} and \Cref{theorem:partialSSP}.  We focus on proving \Cref{theorem:mainLDVC}.

Suppose $G = (V,E)$ is the graph  supplied by \Cref{theorem:ASS} on $|V| = n = 2^{\Theta(k^4 \log^3 k)}$ vertices with a biclique partition of size $m = 2^{O(k \log^4 k)}$. We will use $G$ to build a partial concept class  $\mathbb{G} \subseteq \{0,1,\star\}^{V}$.   This construction is simply the dual of the partial concept class of~\cite{AHHM21} in their proof of \Cref{theorem:mainLDVC}. 

Let $\set{B_1,\ldots,B_m}$ be the size-$m$ biclique partition of the edges of $G$. We fix an orientation $B_i = L_i\times R_i$ for each biclique. Define  $\mathbb{G} \subseteq \{0,1,\star\}^{V}$ as follows. For each $i\in [m]$, associate a concept $h_i:V\to\set{0,1,\star}$ to the biclique $B_i$, defined by 
$$h_i(v) = 
\begin{cases}
    0 & \text{if }v\in L_i \\
    1 & \text{if }v \in R_i \\
    \star & \text{otherwise}
\end{cases}.$$

We first observe that the Littlestone dimension of this concept class is at most 2.
\begin{claim}\label{claim:LDofconstruction}
$\LD(\mathbb{G}) \leq 2$.
\end{claim}
\begin{proof}
    We show that $\mathbb{G}$, viewed as a matrix, does not contain 
    $\begin{bmatrix}
1 & 0 \\
1 & 0 
\end{bmatrix}$ as a submatrix and then show that the existence of this submatrix is necessary for having a Littlestone dimension greater than $2$.

If $\begin{bmatrix}
1 & 0 \\
1 & 0 
\end{bmatrix}$ appears in $\mathbb{G}$ as a submatrix, then there exist  $i\neq j$ and $u\neq v\in V(G)$ such that $h_i(v)=h_j(v) = 1$ and $h_i(u) = h_j(u) = 0$. However, this means that $v\in R_i\cap R_j$ and $u\in L_i\cap L_j$, which in turn implies that the edge $\{u,v\}$ is covered  by both $B_i$ and $B_j$, contradicting the assumption that each edge is covered exactly once.

On the other hand, for a class $\HH\subseteq\set{0,1,\star}^\cX$ with Littlestone dimension greater than 2, there exists a  shattered $\cX$-labelled height-3 full binary tree. In particular,  there exists $h,h'\in\HH$ and points $x_{()},x_1,x_{10}$ such that
\[\begin{array}{lll}
    h(x_{()})=1, & h(x_{1})=0, & h(x_{10})=0,\\
    h'(x_{()})=1, & h'(x_{1})=0, & h'(x_{10})=1.
\end{array}\]
This means that the submatrix restricted to the columns $\set{x_{()}, x_{1}}$ and the rows $\set{h,h'}$ is $\begin{bmatrix}
1 & 0 \\
1 & 0 
\end{bmatrix}$. We conclude that $\LD(\mathbb{G}) \leq 2$.

\end{proof}

\begin{proof}[Proof of \Cref{theorem:mainLDVC}]
    Consider the partial concept class $\mathbb{G}\subseteq\{0,1,\star\}^{V}$ above. By \Cref{claim:LDofconstruction}, we have $\LD(\mathbb{G}) \leq 2$. We show that for every disambiguation $\overline{\mathbb{G}}$ of $\mathbb{G}$, we have $\VC(\overline{\mathbb{G}})\geq \Omega(\log\log n)$. The argument here is similar to the proof of \Cref{theorem:alonVCLD}. 

    Consider a disambiguation $\overline{\mathbb{G}}$ of  $\mathbb{G}$. Note that if two columns $u$ and $v$ are identical in $\overline{\mathbb{G}}$, then there is no edge between $u$ and $v$, as otherwise,  some $h_i$ would have assigned $0$ to one of $u$ and $v$ and $1$ to the other.  Therefore, if two columns $u$ and $v$ are identical, we can color the corresponding vertices with the same color.  Consequently, the number of distinct columns in $\overline{\mathbb{G}}$ is at least the chromatic number $\chi(G)\geq 2^{{\Omega}(k^2)}$.  
    By the SSP lemma (\cref{lemma:SSPlemma}), if $\VC(\overline{\mathbb{G}}^\top)\leq d$, then $\overline{\mathbb{G}}$ must have at most $O(m^d)$ distinct columns. Therefore, 
    $$2^{{\Omega}(k^2)}\leq O(m^d).$$
    Substituting $m = 2^{\tilde{O}(k)}$ shows that $d= \tilde{\Omega}(k)$. Finally, 
    \[
    \VC(\overline{\mathbb{G}})\geq \Omega(\log \VC(\overline{\mathbb{G}}^\top)) \geq \Omega(\log k)\geq {\Omega}(\log\log n).
    \]
    This completes the proof of the first part of  \Cref{theorem:mainLDVC}. 
    
    For the second part, we adopt the same construction in the proof of {\cite[Theorem 11]{AHHM21}}.  Let $\HH_\infty$ be a union of disjoint copies of $\HH_n$ over $n\in\N$, each  supported on a domain $\cX_n$ mutually disjoint from others and the partial concepts of $\HH_n$ extend outside of its domain by $\star$. Since any disambiguation $\HH$ of $\HH_\infty$ simultaneously disambiguates all $\HH_n$,  the Sauer-Shelah-Perles lemma implies that $\VC(\HH)$ must be infinite.
\end{proof}

\subsection{Disambiguations via the SOA algorithm (\cref{theorem:LD-disamb})}
This section is dedicated to the proof of \cref{theorem:LD-disamb}. 

\begin{proof} [Proof of  \cref{theorem:LD-disamb}]
    We prove the statement by showing that for every $r,d\in\N$, there exists a partial concept class $\HH_{r,d}$ on $[n]$, where $n=d(2^r+r)$, such that $d\leq \LD(\HH_{r,d})\leq d+1$ and the SOA disambiguation has VC dimension $\geq dr$ and at least $2^{dr}$ distinct rows. The other cases of $n$ follow by trivially extending the domain.

    For any $r,d\in\N$, define 
    \[\calF_{r,d}=\set{F\subseteq [d2^r]:\, \abs{F}=d}.\]
    Note that $\abs{\calF_{r,d}}=\binom{d2^r}{d}\geq 2^{dr}$. We enumerate the sets in $\calF_{r,d}$ as $F_1,\ldots,F_{\binom{d2^r}{d}}$ in the natural order. 

    Next, we define the partial concept class $\HH_{r,d}$ on domain $[d(2^r+r)]$. The class consists of the partial concepts $h_{i,j}$ for $i\in [\binom{d2^r}{d}]$ and $j\in [dr]$ defined as follows:
    \[h_{i,j}(x)=\begin{cases}
        1 &\text{ if }x\in F_i\\
        0 &\text{ if }x\in [d2^r]\setminus F_i\\
        \beta(i,j) & \text{ if }x=d2^r+j\\
        \star &\text{ otherwise}
    \end{cases},\]
    where $\beta(i,j)$ denotes $j$-th bit of the $dr$-bit binary representation of $i$ if $i\in [2^{dr}]$, and $\beta(i,j)=\star$ otherwise.

    We first prove that $d\leq \LD(\HH_{r,d})\leq d+1$. Note that there is a set of $2^d$ indices $I\subseteq [d2^r]$ which
    \[\set{F_i\cap [d]: i\in I}=\mathcal{P}([d]),\]
    therefore $[d]$ can be shattered by $\set{h_{i,1}:\, i\in I}$ and hence $\LD(\HH_{r,d})\geq \VC(\HH_{r,d})\geq d$. On the other hand, note that $\abs{f^{-1}(1)}\leq d+1$ for any $f\in \HH_{r,d}$, which implies that $\LD(\HH_{r,d})\leq d+1$.

    Next, we consider the SOA disambiguation. We claim that $\set{d2^r+1,\ldots,d(2^r+r)}$ is shattered by $\set{h_{i,1}:i\in [2^{dr}]}$. There are no disambiguations for $x\in [d2^r]$. For $x>d2^r$, note that for any $\vec{b}\in\set{0,1}^{x-1}$, either $\HH_{r,d}|_{\vec{b}}=\emptyset$ or
    \[\HH_{r,d}|_{\vec{b}} =\set{h_{i,j}:j\in [dr]},\]
    where $i\in [d2^r]$ such that $F_i=\set{k\in[d2^r]:\,b_k=1}$. We focus on the latter case and restrict to $i\in [2^{dr}]$. There is exactly one $c\in \set{0,1}$ such that $\HH_{r,d}|_{\vec{b}c}\neq \emptyset$, namely $c=\beta(i,x-d2^r)$ and in this case $\HH_{r,d}|_{\vec{b}c}=\set{h_{i,c}}$. This forces the algorithm to disambiguate every function $f$ with $\vec{b}\in\set{0,1}^{x-1}$ by setting $f(x)=h_{i,c}(x)=\beta(i,x-d2^r)$. In this manner, every $h_{i,j}$ is eventually disambiguated into the same total function:
    \[\overline{h_{i,j}}(x) =\begin{cases}
            1 &\text{ if }x\in F_i\\
            0 &\text{ if }x\in [d2^r]\setminus F_i\\
            \beta(i,x-d2^r) & \text{ if }x>d2^r
        \end{cases}.\]
    In particular, for every $i\in [2^{dr}]$, the bit string $(\overline{h_{i,1}}(d2^r+1),\ldots,\overline{h_{i,1}}(d2^r+dr))$ is the $dr$-bit binary representation of $i$. This provides a witness for which $\VC(\overline{\HH_{r,d}}^{\SOA})\geq dr$.
\end{proof}

As an illustration, we provide the matrix representation of $\HH_{1,2}$ and some essential steps of the SOA disambiguation below in \cref{fig1}.

\def\rzero{\color{red}0}
\def\rone{\color{red}1}

\begin{figure*}[bhpt!]
\centering
\begin{subfigure}[t]{0.45\textwidth}
\centering
\begin{tikzpicture}
    \matrix[matrix of math nodes,
        left delimiter=(,right delimiter=),
        nodes={align=center,font=\footnotesize}] (m){
    1&1&0&0&0&~\\
    1&1&0&0&~&0\\
    1&0&1&0&0&~\\
    1&0&1&0&~&1\\
    1&0&0&1&1&~\\
    1&0&0&1&~&0\\
    0&1&1&0&1&~\\
    0&1&1&0&~&1\\
    0&1&0&1&&\\
    0&1&0&1&&\\
    0&0&1&1&&\\
    0&0&1&1&&\\
    };
\end{tikzpicture}
\caption{Matrix representation of $\HH_{1,2}$: all empty spaces are filled with stars}
\end{subfigure}\hfill
\begin{subfigure}[t]{0.45\textwidth}
\centering
\begin{tikzpicture}
    \matrix[matrix of math nodes,
        left delimiter=(,right delimiter=),
        nodes={align=center,font=\footnotesize}] (m){
    1&1&0&0&0&\rzero\\
    1&1&0&0&\rzero&0\\
    1&0&1&0&0&\rone\\
    1&0&1&0&\rzero&1\\
    1&0&0&1&1&\rzero\\
    1&0&0&1&\rone&0\\
    0&1&1&0&1&\rone\\
    0&1&1&0&\rone&1\\
    0&1&0&1&&\\
    0&1&0&1&&\\
    0&0&1&1&&\\
    0&0&1&1&&\\
    };
\begin{scope}[on background layer]
\foreach \ii in {1,3,5,7}
    \node[fit=(m-\ii-5)(m-\ii-6),fill=gray!50, rounded corners, inner sep=0ex] {};
\end{scope} 
\end{tikzpicture}
\caption{$\overline{\HH_{1,2}}^{\SOA}$: the shaded entries indicate where the shattering occurs}
\end{subfigure}
\caption{$\HH_{1,2}$ and its SOA disambiguation}
\label{fig1}
\end{figure*}

\subsection{Small-size refutation of the Alon-Saks-Seymour conjecture (\cref{theorem:ASS})}\label{section:ASS}
In this section, we present the construction of \cref{theorem:ASS} in detail. The starting point is constructing a Boolean function due to \cite{kasparsASS} in query complexity. This Boolean function then goes through several reductions to be converted into a graph, as described below.

We first introduce some basic definitions related to the notion of \emph{certificate complexity}.
Let $f:\{0,1\}^n \to \{0,1\}$ be a Boolean function. For $b\in \{0,1\}$ and an input $x \in f^{-1}(b)$, a partial input $\rho \in \{0,1,\star\}^n$ is called a $b$-certificate if $x$ is consistent with $\rho$ and for every $x' \in \{0,1\}^n$ consistent with $\rho$, we have $f(x')=b$. The size of $\rho$ is   the number of non-$\star$ entries of $\rho$.  Define $\Cer_b(f, x)$ as the smallest size of a $b$-certificate for $x$. The $b$-certificate complexity of $f$, denoted $\Cer_b(f)$, is the maximum of $\Cer_b(f, x)$ over all $x \in f^{-1}(b)$. 

The \emph{unambiguous} $b$-certificate complexity of $f$, denoted $\UC_b(f)$, is the smallest $k$ such that
\begin{enumerate}
\item  Every input $x \in f^{-1}(b)$ has a $b$-certificate $\rho_x$ of size at most $k$;
\item For every $x \neq y$ in $f^{-1}(b)$, we have $\rho_x \neq \rho_y$. 
\end{enumerate} 

The main result of \cite{kasparsASS} is the following separation between $\UC_1$ and $\Cer_0$. 
\begin{theorem}[{\cite[Theorem 1]{kasparsASS}}] \label{thm:counter-func} 
There is a function $f:\{0,1\}^{12 n^4 \log^2 n} \to \{0,1\}$ such that $\UC_1(f) = O(n \log^3 n)$ and $\Cer_0(f) = \Omega(n^2)$. 
\end{theorem}

The next step of the construction is to transform the function separating the certificate complexities $\UC_1$ and $\Cer_0$ into a communication problem. This is achieved by the ``lifting'' trick: given a function $f:\{0,1\}^n \to \{0,1\}$ and a  ``gadget'' function $g:\{0,1\}^k\times\{0,1\}^k \to \{0,1\}$,  we define $f \circ g^n:\{0,1\}^{nk} \times \{0,1\}^{nk} \to \{0,1\}$ as 
$$f \circ g^n([x_1,\ldots,x_n],[y_1,\ldots,y_n])= f(g(x_1,y_1),\ldots,g(x_n,y_n)).$$

For a communication problem $f:\{0,1\}^m \times \{0,1\}^m \to \{0,1\}$ and $b \in \{0,1\}$, let $\Cov_b(f)$ denote the minimum number of $b$-monochromatic rectangles required to cover all the $b$-entries of $f$. We denote by $\UCov_b(f)$ the minimum number of $b$-monochromatic rectangles required to \emph{partition} all the $b$-entries of $f$. The following theorem provides a connection between the communication complexity parameters and the certificate complexity parameters.

\begin{theorem}[{\cite[Theorem~33]{MR3561784}}] \label{thm:lift}
There exists a gadget $g:\{0,1\}^k\times\{0,1\}^k \to \{0,1\}$ with $k= \Omega(\log n)$ such that for every $f:\{0,1\}^n \to \{0,1\}$, we have 
$$\log \Cov_b(f \circ g^n) = \Omega(k\Cer_b(f)).$$ 
\end{theorem}

Note that for every $b \in \{0,1\}$, we have $\log \UCov_b(f \circ g^n) \le 2k \UC_b(f)$. This combined with \cref{thm:lift} allows one to ``lift'' the $\UC_1$ vs $\Cer_0$ separation of \cref{thm:counter-func} into a $\UCov_1$ vs $\Cov_0$ separation. 
\begin{corollary}
\label{cor:CovVsUCov}
There exists a function $f:\{0,1\}^{O(n^4 \log^3 n)} \times \{0,1\}^{O(n^4 \log^3 n)} \to \{0,1\}$ such that 
$$\log \Cov_0(f) = \Omega(n^2) \qquad \text{and} \qquad \log \UCov_1(f) =n \log^4 n.$$
\end{corollary}

Next, we show how to convert these communication parameters to graph parameters of the biclique partition number and chromatic number. 

\begin{lemma}\label{lemma:ASSfromUC}
    Let $h:\set{0,1}^t\times \set{0,1}^t\to \set{0,1}$ be a Boolean function with $\Cov_0(h)= c$ and $\UCov_1(h)=m$. There exists a graph $G=(V, E)$ on at most $2^{2t}$ vertices with $\bp(G) \le m^2$ and  $\chi(G)\geq \sqrt{c}$.
\end{lemma}
\begin{proof}
Define  the graph $G$ with $V \defeq h^{-1}(0)$   as follows. Two vertices $(x,y),(x',y')\in V$ are adjacent in $G$ iff $h(x,y')=1$ or $h(x',y)=1$. By construction, if $\{(x_1,y_1),\ldots,(x_\ell,y_\ell)\} \subseteq V$ is an independent set, then $\{x_1,\ldots,x_\ell\} \times \{y_1,\ldots,y_\ell\}$ is a $0$-monochromatic rectangle for $h$. Thus every proper vertex coloring of $G$ with $\chi(G)$ colors corresponds to a $0$-cover of $h$ with $\chi(G)$ many $0$-monochromatic rectangles. Therefore, $\chi(G) \ge c$.

We next show that there exists a small set of bicliques such that every edge of $E$ is covered at least once and at most twice by these bicliques. Let $h^{-1}(1) = \bigcup_{i=1}^m (A_i \times B_i)$ be a partition of $h^{-1}(1)$ into $m$ many 1-monochromatic rectangles. Note that every $1$-monochromatic rectangle $A_i \times B_i$ corresponds to a biclique $Q_i \defeq S^-_i\times S^+_i$ in $G$, where
\begin{equation*}
\label{eq:bc}
S^-_i\defeq \{(x,y) \in V(G)  : \ x \in A_i\} \text{ and } S^+_i =   \{(x,y) \in V(G)  : \ y \in B_i\}.
\end{equation*}
Notice that each edge $\{(x,y),(x',y')\}$ of $G$ is covered at least once by $Q_1,\ldots,Q_m$, and it is covered at most twice, the latter happening when $h(x,y')=h(x',y)=1$.

We have thus constructed a graph $G$ on at most $2^{2t}$ vertices such that $\chi(G)\geq c$, and there are at most $m$ bicliques where every edge in $G$ appears in at least one and at most two bicliques. 

Define $H_2$ as the subgraph of $G$ that consists of all the edges covered by exactly two bicliques among $Q_1,\dots, Q_m$. For every $i,j\in [m]$, define $Q_{ij}= (S_i^-\cap S_j^+)\times (S_i^+\cap S_j^-)$. Note that each $Q_{ij}$ is a biclique of $H_2$, and moreover, each edge of $H_2$ appears in exactly one $Q_{ij}$. Hence,  the biclique partition number of $H_2$ is at most $m^2$. 
Now, if $\chi(H_2)\geq \sqrt{c}$, we obtain $H_2$ as the desired graph. Suppose otherwise that $\chi(H_2)<\sqrt{c}$, and consider a proper vertex coloring of $H_2$ with $\sqrt{c}$ colors with color classes $V_1,\dots, V_{\sqrt{c}}$. Since $\chi(G)\geq c$, there must exist $i$ such that the induced subgraph of $G$ on $V_i$, denoted by $G[V_i]$,  satisfies $\chi(G[V_i])\geq \sqrt{c}$. Since $V_i$ is an independent set of $H_2$, thus the restrictions of bicliques $Q_1,\dots, Q_m$ to $V_i$ form a biclique partition of $G[V_i]$. 
\end{proof}

\cref{lemma:ASSfromUC,cor:CovVsUCov} together imply \cref{theorem:ASS}.

\begin{remark} In addition to providing effective bounds on the size of the graph, \cref{lemma:ASSfromUC} also simplifies the original chain of reductions utilized in prior work~\cite{kasparsASS, MR3473357, Bousquet2014, MR1135472} toward achieving a super-polynomial separation between the biclique partition and chromatic numbers.  We will briefly describe the original proof below and highlight the differences. 

\begin{itemize}
\item[(i)] Similar to our proof of \cref{theorem:ASS}, the chain of reduction begins with the function $f$ provided by \cref{cor:CovVsUCov}, such that  $$\log \Cov_0(f) = \Omega(n^2) \qquad \text{and} \qquad \log \UCov_1(f) =n \log^4 n.$$ 

\item[(ii)] Yannakasis~\cite{MR1135472} (see also~\cite[Figure 1]{MR3473357}) showed how to use $f$ to construct a graph $F$  on $\UCov_1(f)= 2^{O(n\log^4 n)}$ vertices such that every Clique-Stable set separator of $F$ is of size at least $\Cov_0(f)=2^{\Omega(n^2)}$. Here, a Clique-Stable set separator is a collection of cuts in $F$ such that for every  disjoint pair $(C,I)$ of a clique $C$ and a stable set $I$ in $F$, there is a cut $(A,B)$ in the collection with $C \subseteq A$ and $I \subseteq B$.

\item[(iii)] Bousquet et. al., \cite[Lemma 23]{Bousquet2014} show how to use $F$ to construct a new graph $G$ with the so-called oriented biclique packing number at most $2^{n\log^4 n}$ and chromatic number $\chi(G)\geq 2^{\Omega(n^2)}$. 

\item[(iv)] The graph $G$ is then turned into a separation between the biclique partition number and chromatic number in a different graph $H$ via a final reduction in \cite{Bousquet2014}. 
\end{itemize}

The above chain of reductions is not sufficient for our application because the graph $G$ of Step (iii) has a vertex for each pair $(C,I)$ of a clique $C$ and a stable set $I$ of $F$, and as a result, there are no effective upper-bounds on the number of vertices of $G$. Our proof of \cref{theorem:ASS} bypasses Step (ii) and employs a more direct approach to construct a small-size graph $G$ that has similar properties to the graph $G$ of Step (iii). 
\end{remark}

\section{Concluding remarks}
A few natural questions remain unanswered. The first question is whether a similar example $\HH$ for \cref{theorem:mainLDVC} with the stronger assumption $\LD(\HH)=1$ exists.
\begin{problem}
Let $\HH$ be a partial class with $\LD(\HH) =1$. Does there exist a disambiguation of $\HH$ by a total class $\barHH$ such that $\LD(\barHH) < \infty$? Is there one with $\VC(\barHH) < \infty$?
\end{problem}

\cref{theorem:mainPolygrowth} shows that for partial classes, having polynomial growth is not a sufficient condition for PAC learnability. A natural candidate reinstatement of the theorem is to work with the more restrictive assumption of linear growth.
\begin{problem}
       Let $\HH\subseteq\{0,1,\star\}^\mathcal{X}$ have polynomial growth with parameter $1$.   Does there exist a disambiguation $\barHH$ of $\HH$ with $\VC( \barHH ) < \infty$?
\end{problem}

Another question is whether one can improve the lower bound of $\Omega(\log \log n)$ in \Cref{theorem:mainLDVC} to $\Omega(\log n)$.

\begin{problem}
    Can the lower bound in \Cref{theorem:mainLDVC} be improved to $\VC(\barHH) \geq \Omega(\log n)$?
\end{problem}

\paragraph{Forbidding combinatorial patterns.}
A natural method to prove upper bounds on the $\VC$ dimension of a concept class is establishing that it does not contain a specific combinatorial pattern. For example, the construction for \Cref{theorem:alonVCLD} in \cite{AHHM21} utilized the fact that the concept class (viewed as a matrix) does not contain the combinatorial patterns 
$\begin{bmatrix}
1 & 1 \\
0 & 0 
\end{bmatrix}$
and 
$\begin{bmatrix}
1 & 0 \\
0 & 1 
\end{bmatrix}$, which are patterns that are in   any concept class $\HH$ with $\VC(\HH) \geq 2$. 
Similarly, the dual construction in \Cref{theorem:mainLDVC} forbids the pattern 
$\begin{bmatrix}
1 & 0 \\
1 & 0 
\end{bmatrix}$, a compulsory pattern for any concept class $\HH$ with $\LD(\HH)\geq 3$.  

\begin{problem}\label{problem:pattern}
    Suppose $\HH\subseteq \{0,1,\star\}^{[n]}$ does not contain the pattern $
    \begin{bmatrix}
1 & 1 \\
0 & 1 
\end{bmatrix}$. Does every disambiguation $\barHH$ of $\HH$ satisfy $\VC(\barHH)= O(1)$?
\end{problem}

\paragraph{Acknowledgement} 
We wish to thank Mika G{\"o}{\"o}s for clarifying the reductions in \cite{Bousquet2014,MR3473357,MR3561784,kasparsASS}. 

\bibliographystyle{alpha}
\bibliography{ref}
	
\appendix

\end{document}